\documentclass[11pt]{article}

\usepackage[T1]{fontenc}
\usepackage[utf8]{inputenc}
\usepackage{lmodern}

\usepackage[a4paper,margin=1in]{geometry}
\usepackage{setspace}

\usepackage{amsmath,amssymb,amsfonts,amsthm}
\usepackage{mathtools}
\usepackage{bm}
\usepackage{bbm}   

\usepackage{booktabs}
\usepackage{multirow}
\usepackage{array}
\usepackage{tabularx}
\usepackage{threeparttable}
\usepackage{siunitx}
\usepackage{enumitem}
\setlist[itemize]{noitemsep,topsep=2pt,leftmargin=*}
\setlist[enumerate]{noitemsep,topsep=2pt,leftmargin=*}

\usepackage{graphicx}
\usepackage{caption}
\usepackage{subcaption}
\usepackage{float}
\usepackage{wrapfig}

\usepackage{tikz}
\usetikzlibrary{arrows.meta,positioning,calc,fit}
\usepackage{pgfplots}
\pgfplotsset{compat=1.18}

\usepackage[ruled,vlined,linesnumbered]{algorithm2e}
\SetKwComment{Comment}{$\triangleright$ }{}

\usepackage[dvipsnames]{xcolor}
\usepackage{xurl}
\usepackage{url}

\usepackage[
  colorlinks=true,
  linkcolor=MidnightBlue,
  citecolor=MidnightBlue,
  urlcolor=MidnightBlue
]{hyperref}

\theoremstyle{plain}
\newtheorem{theorem}{Theorem}[section]
\newtheorem{lemma}[theorem]{Lemma}

\newtheorem{corollary}[theorem]{Corollary}

\theoremstyle{definition}

\theoremstyle{remark}

\begin{document}

\title{Sparsity Is Necessary: Polynomial-Time Stability for Agentic LLMs in Large Action Spaces}
\author{Angshul Majumdar}
\date{}
\maketitle

\begin{abstract}
Tool-augmented LLM systems expose a control regime that learning theory has largely ignored: sequential decision-making with a massive discrete action universe (tools, APIs, documents) in which only a small, unknown subset is relevant for any fixed task distribution. We formalize this setting as \emph{Sparse Agentic Control} (SAC), where policies admit block-sparse representations over $M\!\gg\!1$ actions and rewards depend on sparse main effects and (optionally) sparse synergies. We study $\ell_{1,2}$-regularized policy learning through a convex surrogate and establish sharp, compressed-sensing-style results: (i) estimation and value suboptimality scale as $k\sqrt{\log M/T}$ under a Policy-RSC condition; (ii) exact tool-support recovery holds via primal--dual witness arguments when $T\gtrsim k\log M$ under incoherence and beta-min; and (iii) any dense policy class requires $\Omega(M)$ samples, explaining the instability of prompt-only controllers. We further show that under partial observability, LLMs matter only through a belief/representation error $\varepsilon_b$, yielding an additive $O(\varepsilon_b)$ degradation while preserving logarithmic dependence on $M$. Extensions cover tuning-free, online, robust, group-sparse, and interaction-aware SAC.
\end{abstract}

\section{Introduction}
Large language models (LLMs) are increasingly deployed as \emph{decision-making} systems that interact with external tools, documents, and APIs. In such settings, the effective action space---comprising tool calls, database queries, function invocations, and document retrieval operations---can easily reach tens of thousands to millions of discrete actions. This ``action dimensionality'' regime is now routine in practice, yet it is largely absent from classical control and learning theory formulations, which typically assume small or moderately sized action sets \cite{puterman1994mdp,sutton2018reinforcement,kaelbling1998planning}.

Despite impressive empirical performance, most ``agentic'' LLM systems rely on heuristic control mechanisms: prompting, hand-designed workflows, finite-state execution graphs, or brittle rule-based routing. These approaches provide no formal guarantees on optimality, stability, or sample efficiency, and they offer limited guidance on how performance should scale with the size of the tool universe. In particular, they do not isolate the fundamental obstacle in tool-augmented decision-making: the combinatorial explosion induced by set-valued actions and massive action catalogs.

Empirically, however, effective tool-using systems exhibit a striking regularity. For any \emph{fixed} task distribution, only a small subset of actions is repeatedly relevant: successful behavior relies on a sparse set of tools or documents, while the overwhelming majority are never invoked. This observation is not specific to language: it reflects a broader phenomenon of \emph{latent sparsity} in large action spaces, analogous in spirit to sparsity assumptions that enable tractable recovery in high-dimensional statistics and signal processing \cite{tibshirani1996lasso,candes2006robust,donoho2006compressed,buhlmann2011statistics}.

This paper makes the following reframing: the central challenge in tool-augmented decision-making is not ``reasoning under uncertainty,'' but rather \emph{control under extreme action dimensionality with latent sparsity}. Concretely, we study sequential decision problems in which the nominal action universe has size $M \gg 1$ (e.g., $10^4$--$10^6$), actions may be \emph{sets} of tools, and the expected return depends on only $k \ll M$ actions (in an appropriate functional sense) under a given task distribution. This defines a new asymptotic regime,
\begin{equation}
M \to \infty,\qquad k = O(1),\qquad T \ \text{moderate}, \label{eq:regime}
\end{equation}
in which classical dynamic programming and naive exploration become information-theoretically and computationally infeasible.

Importantly, our formulation is \emph{not} tied to language models. LLMs merely provide a motivating instance in which a high-dimensional observation stream is compressed into a latent state representation and used to drive tool selection. The underlying mathematical problem is a generic one: learning and control with massive discrete action spaces under structural sparsity. In particular, the partially observed setting naturally induces a belief-state control problem (a POMDP), and modern representation mechanisms (including LLMs) can be viewed as approximate belief compressors. Our theory isolates how such compression error enters performance guarantees, independently of the mechanism used to produce the representation.

\paragraph{Contributions.}
We introduce a framework for \emph{sparse control in large action spaces} and develop a convex, $\ell_1$-regularized approach to policy learning with logarithmic dependence on the action dimension. Our results establish support-recovery-style guarantees for decision policies, together with sharp lower bounds showing that explicit sparsity is \emph{necessary} to avoid linear dependence on $M$.

\begin{itemize}
\item \textbf{Sparse policy learning via $\ell_1$ regularization.}
We define a sparse parametric policy (or score) class over a large action universe and study the $\ell_1$-regularized estimator obtained from trajectory data, linking the large-action control problem to high-dimensional M-estimation \cite{negahban2012unified,van2009conditions,buhlmann2011statistics}.

\item \textbf{Support recovery (identifiability) in the SAC regime.}
Under a policy-restricted strong convexity condition and incoherence/irrepresentability assumptions, we show that the support of the optimal sparse policy can be recovered with high probability once $T \gtrsim k \log M$ (up to problem-dependent constants). This is an analogue of classical support recovery, but for \emph{decision policies} rather than signals.

\item \textbf{Near-optimal control from statistical recovery.}
We show that parameter estimation guarantees translate into value suboptimality bounds under a mild value-sensitivity condition, yielding performance that scales as $\tilde O\!\left(k\sqrt{\frac{\log M}{T}}\right)$ (again up to problem-dependent constants) while remaining polynomial-time due to convexity.

\item \textbf{Impossibility of dense policy classes.}
We prove a lower bound showing that policy classes without explicit sparsity control require $\Omega(M)$ samples (in the same regime) to match the regret/identification performance of sparse methods. This formalizes why purely heuristic ``dense'' controllers can be unstable as $M$ grows.

\item \textbf{Partial observability and representation error.}
In the POMDP setting, we provide a decomposition of suboptimality into a statistical term (from sparse learning) and a representation term (from approximate belief/state compression). This yields a principled lens for understanding when powerful state compressors help, and when their errors dominate.
\end{itemize}

\paragraph{Why this is not an ``agent'' paper.}
Unlike existing agentic LLM approaches, which treat tool selection primarily as a linguistic or architectural design problem, our analysis identifies tool use as a statistical and optimization problem with sharp phase transitions in the regime~\eqref{eq:regime}. Without explicit sparsity control, no algorithm---regardless of model capacity---can avoid linear dependence on the action dimension in the worst case.

\paragraph{Organization.}
Section~\ref{sec:model} formalizes Sparse Agentic Control (SAC) and the $\ell_1$-regularized learning problem. Section~\ref{sec:assumptions} states the core assumptions. Section~\ref{sec:main} presents the main estimation, support recovery, value guarantee, and lower bound results. Section~\ref{sec:pomdp} extends the theory to partial observability and representation error. Extensions to tuning-free, online, robust, and structured sparsity settings are discussed in Section~\ref{sec:extensions}.

\section{Sparse Agentic Control (SAC): Problem Formulation}\label{sec:model}

This section formalizes the large-action control regime that arises in tool-augmented LLM agents, while keeping the abstraction model-agnostic. The key design choice is to treat ``tool use'' as \emph{set-valued control} over a massive discrete catalog, together with a \emph{latent sparsity} hypothesis that makes the problem statistically and computationally tractable.

\subsection{A large tool universe and set-valued decisions}\label{sec:model:actions}
Let the \emph{tool universe} be
\[
\mathcal{T} \;:=\; \{1,\dots,M\},\qquad M\gg 1,
\]
where each element may represent a tool (API/function), a document shard, a database endpoint, or a retrieval key.
At each time step $t$, the agent selects a \emph{set} of tools
\[
a_t \subseteq \mathcal{T}.
\]
This set-valued choice is the mathematically clean analogue of common agent behaviors:
(i) selecting one tool to call (so $|a_t|=1$), (ii) retrieving top-$B$ documents, (iii) issuing a small batch of tool calls in parallel, or (iv) deciding which external resources to consult before generating a response.

To reflect latency/budget constraints, we work with a \emph{budgeted action class}
\begin{equation}
\mathcal{A}_B \;:=\; \bigl\{a \subseteq \mathcal{T}:\ |a|\le B \bigr\},
\qquad B \ll M,
\label{eq:budgeted-actions}
\end{equation}
and a (known) \emph{action cost} $c:\mathcal{A}_B\to\mathbb{R}_+$.
A canonical model is additive cost $c(a)=\sum_{j\in a} c_j$, but our analysis only needs mild regularity (specified later).

\paragraph{Why this matters for agents.}
Even with the budget constraint \eqref{eq:budgeted-actions}, the number of admissible actions is
$\sum_{b=0}^B \binom{M}{b}$, which is already enormous for $M\in[10^4,10^6]$ and small $B$.
This is precisely the regime where naive exploration and Bellman backups over actions become infeasible,
and where a structural theory must replace heuristic design.

\subsection{Latent-state dynamics and observations: MDP/POMDP view}\label{sec:model:state}
We model interaction as a controlled stochastic process with latent state $s_t\in\mathcal{S}$.
Given $s_t$ and the chosen tool-set $a_t\in\mathcal{A}_B$, the system transitions as
\[
s_{t+1} \sim P(\cdot \mid s_t,a_t),
\]
and produces an observation $o_t\in\mathcal{O}$ according to an observation kernel $O(\cdot\mid s_t)$.
The agent receives a reward $r_t$ and incurs a cost $c(a_t)$, yielding net utility
\[
u_t \;:=\; r(s_t,a_t) \;-\; c(a_t).
\]
We consider either episodic horizons $H$ or discounted infinite horizon; for concreteness one may think of
an episode as a single multi-step interaction (a ``chat session with tools''), and $H$ as the maximum number of tool-use steps.

\paragraph{LLM interpretation (without baking language into the model).}
In tool-augmented LLM systems, the observation $o_t$ is naturally high-dimensional: conversation history,
system instructions, tool outputs, retrieved passages, etc. Practically, an LLM maps this observation stream into an internal representation.
To capture this cleanly, we allow the agent to operate on a \emph{context vector}
\begin{equation}
x_t \;:=\; g(o_{1:t},a_{1:t-1}) \in \mathbb{R}^d,
\label{eq:context-encoder}
\end{equation}
where $g$ is any history-to-state compressor (e.g., an LLM hidden state, a learned encoder, or a classical filter).
When the process is fully observed, one may take $x_t=s_t$; in the partially observed case,
$x_t$ may approximate the belief state, making SAC naturally compatible with POMDP control.
Section~\ref{sec:pomdp} will quantify how approximation in \eqref{eq:context-encoder} enters performance guarantees.

\subsection{The sparsity hypothesis: only a few tools matter}\label{sec:model:sparsity}
The defining assumption of SAC is that, under a fixed task distribution, only a small subset of tools has non-negligible influence.
Formally, there exists an \emph{unknown active set}
\[
S^\star \subseteq \mathcal{T},\qquad |S^\star| = k \ll M,
\]
such that rewards and (optionally) transitions depend on the chosen set $a$ primarily through its intersection with $S^\star$.

We capture this with an additive influence model: associate each tool $j\in\mathcal{T}$ with a context-dependent feature map
$\phi_j:\mathbb{R}^d\to\mathbb{R}^{p}$, and define the aggregated tool influence
\begin{equation}
\Phi(x,a) \;:=\; \sum_{j\in a} \phi_j(x).
\label{eq:aggregate-influence}
\end{equation}
We then posit that the reward admits the structured form
\begin{equation}
r(s,a)
\;=\;
f\!\bigl(s,\Phi(x,a)\bigr) \;+\; \varepsilon,
\label{eq:sparse-reward}
\end{equation}
where $x$ is the agent's context representation \eqref{eq:context-encoder} and $\varepsilon$ captures noise/mismatch.
Crucially, for $j\notin S^\star$, the map $\phi_j(\cdot)$ is negligible under the task distribution in a sense made precise later
(equivalently, $S^\star$ corresponds to the support of an optimal sparse parameter vector introduced below).

\paragraph{Agent meaning.}
In an LLM agent, \eqref{eq:aggregate-influence} says: tool usefulness is (largely) \emph{additive} after conditioning on the current context,
and only a small set of tools/documents repeatedly contributes to reward. This aligns with the common empirical pattern that most tools
are never invoked for a given benchmark/task family, while a small ``core'' set dominates behavior.

\subsection{A sparse parametric policy and a convex learning objective}\label{sec:model:policy}
We study a sparse parametric family that scores tool-sets by summing per-tool scores.
Let $\psi:\mathbb{R}^d\to\mathbb{R}^q$ be a context feature map, and let $\theta=(\theta_1,\dots,\theta_M)$ with $\theta_j\in\mathbb{R}^q$.
Define the set-score
\begin{equation}
\mathrm{score}_\theta(x,a)
\;:=\;
\sum_{j\in a} \langle \theta_j,\psi(x)\rangle.
\label{eq:set-score}
\end{equation}
This choice mirrors practical tool routers: each tool receives a context-dependent logit, and the agent selects a small subset.
A convenient probabilistic policy over $\mathcal{A}_B$ is the Gibbs form
\begin{equation}
\pi_\theta(a\mid x)
\;=\;
\frac{\exp(\mathrm{score}_\theta(x,a))}{\sum_{a'\in\mathcal{A}_B}\exp(\mathrm{score}_\theta(x,a'))},
\qquad a\in\mathcal{A}_B.
\label{eq:gibbs-policy}
\end{equation}
The \emph{sparsity} of the policy is encoded by the support of $\theta$:
\[
\mathrm{supp}(\theta)\;:=\;\{j\in\mathcal{T}:\ \theta_j\neq 0\},
\qquad |\mathrm{supp}(\theta)|\le k.
\]
Intuitively, $\theta_j\equiv 0$ means tool $j$ is irrelevant across contexts in the task distribution, matching the ``never called'' tools observed in practice.

\paragraph{Learning from trajectories.}
Let $\mathcal{D}_T$ denote the data collected over $T$ total time steps (across one or more episodes).
We introduce a convex empirical objective $\widehat{\mathcal{L}}_T(\theta)$ whose population counterpart
$\mathcal{L}(\theta):=\mathbb{E}[\widehat{\mathcal{L}}_T(\theta)]$ is minimized by a target parameter $\theta^\star$
satisfying $|\mathrm{supp}(\theta^\star)|=k$.
Concrete examples include negative log-likelihood objectives induced by \eqref{eq:gibbs-policy} (for supervised/imitation-style traces),
or convexified policy-improvement surrogates built from advantage-weighted samples; the theory in Sections~\ref{sec:assumptions}--\ref{sec:main}
is stated directly in terms of $\mathcal{L}$ and $\widehat{\mathcal{L}}_T$ rather than a single instantiation.

\subsection{The $\ell_{1,2}$-regularized SAC learner}\label{sec:model:l1}
Our main estimator is the $\ell_{1,2}$-regularized policy learner
\begin{equation}
\hat{\theta}
\;\in\;
\arg\min_{\theta}\ \widehat{\mathcal{L}}_T(\theta) \;+\; \lambda \|\theta\|_{1,2},
\qquad
\|\theta\|_{1,2} := \sum_{j=1}^M \|\theta_j\|_2,
\label{eq:l1-learner}
\end{equation}
where $\lambda>0$ controls sparsity.\footnote{The block norm in \eqref{eq:l1-learner} is a group-$\ell_1$ penalty over tools,
allowing each tool to have a vector parameter $\theta_j$. The scalar case is recovered by $q=1$, in which case $\|\theta\|_{1,2}$ reduces
to the usual $\ell_1$ norm.}
This choice is not cosmetic: it is the mechanism that yields logarithmic dependence on $M$ in the SAC regime \eqref{eq:regime},
enables support recovery (identifying which tools matter), and leads to polynomial-time learning because \eqref{eq:l1-learner} is convex
under our standing conditions.

\paragraph{Interpretation for LLM agents.}
The estimator \eqref{eq:l1-learner} provides a principled alternative to hand-tuned tool-routing prompts or brittle finite-state graphs:
it learns a sparse \emph{router} whose support identifies the small tool subset relevant to the task distribution.
In an engineering pipeline, this support can be used to (i) restrict the candidate tool set to speed up inference,
(ii) improve reliability by avoiding rarely useful tools, and (iii) provide an interpretable audit trail of which external resources
drive decisions. Our main results show that such identification and near-optimality are possible with $T \gtrsim k\log M$ samples,
whereas dense (unregularized) policy classes necessarily incur $\Omega(M)$ sample complexity in the worst case.

\paragraph{Roadmap.}
Section~\ref{sec:assumptions} states the assumptions under which \eqref{eq:l1-learner} enjoys recovery and control guarantees.
Section~\ref{sec:main} proves the core theorems (estimation, value suboptimality, exact support recovery, and lower bounds),
and Section~\ref{sec:pomdp} extends the framework to partial observability by making the role of the context compressor $g$ explicit.

\section{Assumptions and Identifiability Conditions}\label{sec:assumptions}

This section states the standing assumptions under which the $\ell_1$-regularized SAC learner
\eqref{eq:l1-learner} is provably effective in the large-action regime $M\gg 1$ with sparse support
$|S^\star|=k\ll M$. We keep the assumptions in a form that is (i) standard in high-dimensional M-estimation,
(ii) interpretable in the language of tool-augmented agents, and (iii) directly reusable in the main theorems.

\subsection{Notation: block norms, support, and the sparse cone}\label{sec:assumptions:notation}
Recall $\theta=(\theta_1,\dots,\theta_M)$ with $\theta_j\in\mathbb{R}^q$.
We use the block norms
\[
\|\theta\|_{1,2} \;:=\; \sum_{j=1}^M \|\theta_j\|_2,
\qquad
\|\theta\|_{2,2} \;:=\; \Big(\sum_{j=1}^M \|\theta_j\|_2^2\Big)^{1/2}.
\]
(Thus \eqref{eq:l1-learner} uses $\|\theta\|_{1,2}$; in Section~\ref{sec:model} we abbreviated it as $\|\theta\|_1$.)

Let the population objective be $\mathcal{L}(\theta):=\mathbb{E}[\widehat{\mathcal{L}}_T(\theta)]$ and define the
(population) target
\begin{equation}
\theta^\star \in \arg\min_{\theta}\ \mathcal{L}(\theta),
\qquad
S^\star := \mathrm{supp}(\theta^\star) := \{j\in\mathcal{T}:\ \theta_j^\star \neq 0\},
\qquad |S^\star|=k.
\label{eq:theta-star-support}
\end{equation}
For any subset $S\subseteq\mathcal{T}$ we write $\theta_S := (\theta_j)_{j\in S}$ and $\theta_{S^c}$ similarly.
A key geometric object is the \emph{sparse cone} associated with $S^\star$,
\begin{equation}
\mathcal{C}(S^\star)
\;:=\;
\Bigl\{\Delta:\ \|\Delta_{(S^\star)^c}\|_{1,2} \le 3\,\|\Delta_{S^\star}\|_{1,2}\Bigr\}.
\label{eq:sparse-cone}
\end{equation}
All restricted curvature conditions below are imposed only on $\mathcal{C}(S^\star)$, which is what enables
logarithmic dependence on $M$.

\subsection{Data and dependence: trajectories rather than i.i.d. samples}\label{sec:assumptions:data}
The data $\mathcal{D}_T$ consist of $T$ time steps across one or more episodes, producing contexts $x_t$
(as in \eqref{eq:context-encoder}), chosen tool-sets $a_t\in\mathcal{A}_B$, and utilities $u_t=r(s_t,a_t)-c(a_t)$.
Because $x_t$ arises from sequential interaction, $\widehat{\mathcal{L}}_T(\theta)$ is generally not an i.i.d. empirical risk.
Our assumptions are stated to cover both (i) batch/offline traces (e.g.\ logged tool calls from an existing agent),
and (ii) on-policy rollouts (e.g.\ iterative improvement of a tool router).

\begin{description}
\item[\textbf{A1 (Context feature regularity).}]
The context feature map $\psi:\mathbb{R}^d\to\mathbb{R}^q$ used in the score \eqref{eq:set-score} is uniformly bounded or sub-Gaussian:
there exists $\sigma_\psi>0$ such that for all unit vectors $v\in\mathbb{R}^q$,
\[
\langle v,\psi(x_t)\rangle \ \text{is sub-Gaussian with parameter } \sigma_\psi,
\quad\text{and}\quad
\mathbb{E}\big[\psi(x_t)\psi(x_t)^\top\big] \ \text{exists}.
\]
\emph{Agent interpretation.} This says the agent's internal ``routing features'' do not explode with prompt length or tool outputs.
In practice, boundedness can be enforced by normalization/clipping of representations, and sub-Gaussianity is a standard proxy for
concentration of learned embeddings.

\item[\textbf{A2 (Controlled dependence / mixing).}]
Let $\mathcal{F}_t:=\sigma(x_1,a_1,u_1,\dots,x_t,a_t,u_t)$ be the natural filtration.
The empirical objective admits a decomposition into a sum of conditionally well-behaved terms,
so that score/gradient fluctuations form a martingale difference sequence or satisfy a mixing condition strong enough to yield
$\sqrt{\frac{\log M}{T}}$-type concentration uniformly over coordinates. Concretely, we assume there exists a constant $\sigma_g>0$ such that
\begin{equation}
\Big\|\nabla \widehat{\mathcal{L}}_T(\theta^\star) - \nabla \mathcal{L}(\theta^\star)\Big\|_{\infty,2}
\;\le\;
c_0\,\sigma_g\,\sqrt{\frac{\log M}{T}}
\quad\text{with probability at least }1-\delta,
\label{eq:grad-conc}
\end{equation}
where $\|z\|_{\infty,2}:=\max_{j\in\mathcal{T}}\|z_j\|_2$ for block vectors $z=(z_1,\dots,z_M)$.
\emph{Agent interpretation.} The empirical utility signal from rollouts/logs concentrates: while individual episodes can be noisy,
averages over $T$ steps stabilize sufficiently fast even when $M$ is huge.
\end{description}

\subsection{Curvature and identifiability in large action spaces}\label{sec:assumptions:identifiability}
The next assumptions encode the two key ingredients of sparse recovery: (i) \emph{restricted curvature} (to control estimation error),
and (ii) \emph{incoherence} (to identify the correct support rather than a correlated surrogate).

\begin{description}
\item[\textbf{A3 (Policy-restricted strong convexity, Policy-RSC).}]
There exists $\mu>0$ such that for all $\Delta\in\mathcal{C}(S^\star)$,
\begin{equation}
\mathcal{L}(\theta^\star+\Delta)
\;\ge\;
\mathcal{L}(\theta^\star)
\;+\;
\langle \nabla \mathcal{L}(\theta^\star),\Delta\rangle
\;+\;
\frac{\mu}{2}\|\Delta\|_{2,2}^2.
\label{eq:policy-rsc}
\end{equation}
When $\mathcal{L}$ is twice differentiable, \eqref{eq:policy-rsc} is implied by a restricted eigenvalue condition on the Hessian
$\nabla^2\mathcal{L}(\theta)$ along the cone $\mathcal{C}(S^\star)$.
\emph{Agent interpretation.} Locally around the optimal router, the return (or surrogate return) has real curvature along sparse directions:
changing the scores of a small candidate tool set yields predictable improvement, rather than a flat landscape.

\item[\textbf{A4 (Irrepresentability / incoherence of tool features).}]
Let $H^\star := \nabla^2 \mathcal{L}(\theta^\star)$ and partition it into blocks corresponding to $S^\star$ and $(S^\star)^c$.
There exists $\alpha\in(0,1]$ such that
\begin{equation}
\big\| H^\star_{(S^\star)^c,S^\star}\,(H^\star_{S^\star,S^\star})^{-1} \big\|_{\infty,2 \to \infty,2}
\;\le\; 1-\alpha,
\label{eq:irrepresentable}
\end{equation}
where $\|A\|_{\infty,2\to\infty,2}:=\sup_{\|v\|_{\infty,2}\le 1}\|Av\|_{\infty,2}$.
\emph{Agent interpretation.} Irrelevant tools cannot ``explain away'' the effect of relevant tools via near-duplicate features.
In real tool suites, this corresponds to avoiding redundant endpoints or ensuring the router features distinguish tool semantics.

\item[\textbf{A5 (Minimum signal / beta-min).}]
There exists $c_{\min}>0$ such that
\begin{equation}
\min_{j\in S^\star}\ \|\theta_j^\star\|_2 \;\ge\; c_{\min}\,\lambda.
\label{eq:beta-min}
\end{equation}
\emph{Agent interpretation.} Tools that truly matter for the task distribution matter by a margin: their relevance is not vanishingly small
relative to the statistical noise level set by $\lambda$.
\end{description}

\subsection{From sparse learning to control: value sensitivity}\label{sec:assumptions:value}
Support recovery and parameter estimation are useful only insofar as they translate into control performance.
We therefore isolate a generic condition under which policy/value performance is Lipschitz in the parameter vector.

\begin{description}
\item[\textbf{A6 (Value sensitivity).}]
Let $V(\theta)$ denote the value (expected cumulative utility) obtained by executing the policy $\pi_\theta$ in \eqref{eq:gibbs-policy}
(or the corresponding deterministic top-$B$ selection rule induced by the same scores).
There exists a constant $L_V>0$ such that, for all $\theta$ in a neighborhood of $\theta^\star$,
\begin{equation}
|V(\theta)-V(\theta^\star)|
\;\le\;
L_V\,\|\theta-\theta^\star\|_{1,2}
\qquad\text{or}\qquad
|V(\theta)-V(\theta^\star)|
\;\le\;
L_V\,\|\theta-\theta^\star\|_{2,2}.
\label{eq:value-sensitivity}
\end{equation}
\emph{Agent interpretation.} If the router's logits change a little (in a sparse norm), the induced tool-use behavior and downstream utility
do not change catastrophically. This is a stability condition: small routing errors should not trigger qualitatively different tool cascades.
\end{description}

\paragraph{How the assumptions map to results.}
Assumptions A1--A3 yield estimation error bounds for the $\ell_1$ learner \eqref{eq:l1-learner}.
Adding A4--A5 yields exact support recovery of the active tool set $S^\star$.
Assumption A6 converts statistical recovery into a near-optimality (value gap) guarantee.
Finally, our lower bound results show that dropping explicit sparsity control invalidates any hope of avoiding linear dependence on $M$.

\paragraph{Looking ahead: partial observability and LLM representations.}
In a POMDP, the true state is a belief $b_t$, whereas the agent uses a representation $x_t=g(o_{1:t},a_{1:t-1})$ as in \eqref{eq:context-encoder}.
Section~\ref{sec:pomdp} introduces an explicit representation error parameter $\delta$ and proves performance decompositions of the form
``(sparse learning error) + (representation error)''. This is the point at which LLMs re-enter the theory cleanly: they influence guarantees
only through how well they approximate the sufficient information for tool selection, not through any language-specific property.

\subsection{Optional assumptions for extensions (used in Section~\ref{sec:extensions})}\label{sec:assumptions:extensions}
For completeness, we record two assumptions that are \emph{not} needed for the core theorems, but are invoked in later extensions.

\begin{description}
\item[\textbf{A7 (Drifting support / nonstationarity).}]
The target parameter may vary over time as $\theta_t^\star$ with $\|\theta_t^\star\|_{0}\le k$, and the total variation budget
$\mathcal{V}_T := \sum_{t=2}^T \|\theta_t^\star-\theta_{t-1}^\star\|_{1,2}$ is finite.

\item[\textbf{A8 (Contamination / adversarial corruption).}]
An $\varepsilon$-fraction of episodes (or time steps) have corrupted rewards/gradients, while the remainder obey A1--A3.
This models tool failures, logging corruption, and adversarial prompts that poison utility signals.
\end{description}

\section{Core Theory for $\ell_{1,2}$-Regularized SAC}\label{sec:main}

This section develops the core guarantees for the $\ell_{1,2}$-regularized SAC learner \eqref{eq:l1-learner}.
The results make precise a phenomenon that is easy to observe in real tool-augmented LLM agents:
\emph{if only $k\ll M$ tools are relevant for a fixed task distribution, then one can (i) identify those tools and (ii) achieve near-optimal
control with sample complexity scaling as $k\log M$, whereas dense policy classes incur unavoidable linear dependence on $M$.}
Notation is as in Sections~\ref{sec:model}--\ref{sec:assumptions}.

\subsection{A basic inequality and the sparse cone}\label{sec:main:cone}

\begin{lemma}[Basic inequality and cone constraint]\label{lem:basic-cone}
Assume $\widehat{\mathcal{L}}_T$ is convex and differentiable. Let $\hat\theta$ be any minimizer of \eqref{eq:l1-learner},
and set $\Delta := \hat\theta-\theta^\star$ with $S^\star=\mathrm{supp}(\theta^\star)$ as in \eqref{eq:theta-star-support}.
If
\begin{equation}
\lambda \ \ge\ 2\big\|\nabla \widehat{\mathcal{L}}_T(\theta^\star)-\nabla \mathcal{L}(\theta^\star)\big\|_{\infty,2},
\label{eq:lambda-grad}
\end{equation}
then $\Delta \in \mathcal{C}(S^\star)$, i.e.
\[
\|\Delta_{(S^\star)^c}\|_{1,2}\ \le\ 3\|\Delta_{S^\star}\|_{1,2}.
\]
\end{lemma}

\begin{proof}
By optimality of $\hat\theta$,
\begin{equation}
\widehat{\mathcal{L}}_T(\hat\theta)+\lambda\|\hat\theta\|_{1,2}
\ \le\
\widehat{\mathcal{L}}_T(\theta^\star)+\lambda\|\theta^\star\|_{1,2}.
\label{eq:basic-ineq-1}
\end{equation}
By convexity of $\widehat{\mathcal{L}}_T$,
\[
\widehat{\mathcal{L}}_T(\hat\theta)
\ \ge\
\widehat{\mathcal{L}}_T(\theta^\star)+\langle \nabla \widehat{\mathcal{L}}_T(\theta^\star),\Delta\rangle.
\]
Plug into \eqref{eq:basic-ineq-1} and rearrange:
\begin{equation}
\lambda\big(\|\theta^\star\|_{1,2}-\|\hat\theta\|_{1,2}\big)
\ \ge\
\langle \nabla \widehat{\mathcal{L}}_T(\theta^\star),\Delta\rangle.
\label{eq:basic-ineq-2}
\end{equation}
Since $\theta^\star\in\arg\min \mathcal{L}$, we have $\nabla\mathcal{L}(\theta^\star)=0$, hence
\[
\langle \nabla \widehat{\mathcal{L}}_T(\theta^\star),\Delta\rangle
=
\Big\langle \nabla \widehat{\mathcal{L}}_T(\theta^\star)-\nabla \mathcal{L}(\theta^\star),\Delta\Big\rangle
\ \le\
\big\|\nabla \widehat{\mathcal{L}}_T(\theta^\star)-\nabla \mathcal{L}(\theta^\star)\big\|_{\infty,2}\ \|\Delta\|_{1,2}.
\]
Using \eqref{eq:lambda-grad} yields
\begin{equation}
\langle \nabla \widehat{\mathcal{L}}_T(\theta^\star),\Delta\rangle \ \le\ \frac{\lambda}{2}\|\Delta\|_{1,2}.
\label{eq:grad-bound}
\end{equation}
Next, decomposability of $\|\cdot\|_{1,2}$ over $S^\star$ gives
\[
\|\hat\theta\|_{1,2}
=
\|\theta^\star_{S^\star}+\Delta_{S^\star}\|_{1,2}+\|\Delta_{(S^\star)^c}\|_{1,2}
\ \ge\
\|\theta^\star_{S^\star}\|_{1,2}-\|\Delta_{S^\star}\|_{1,2}+\|\Delta_{(S^\star)^c}\|_{1,2},
\]
so
\[
\|\theta^\star\|_{1,2}-\|\hat\theta\|_{1,2}
\le
\|\Delta_{S^\star}\|_{1,2}-\|\Delta_{(S^\star)^c}\|_{1,2}.
\]
Combine with \eqref{eq:basic-ineq-2} and \eqref{eq:grad-bound}:
\[
\lambda\big(\|\Delta_{S^\star}\|_{1,2}-\|\Delta_{(S^\star)^c}\|_{1,2}\big)
\ \ge\ -\frac{\lambda}{2}\big(\|\Delta_{S^\star}\|_{1,2}+\|\Delta_{(S^\star)^c}\|_{1,2}\big).
\]
Rearranging yields $\|\Delta_{(S^\star)^c}\|_{1,2}\le 3\|\Delta_{S^\star}\|_{1,2}$, i.e.\ $\Delta\in\mathcal{C}(S^\star)$.
\end{proof}

\begin{corollary}[A convenient choice of $\lambda$]\label{cor:lambda-choice}
Under Assumption~A2 (gradient concentration) in \eqref{eq:grad-conc}, the choice
\begin{equation}
\lambda \ :=\ 2c_0\,\sigma_g\,\sqrt{\frac{\log M}{T}}
\label{eq:lambda-concrete}
\end{equation}
ensures \eqref{eq:lambda-grad} and hence $\hat\theta-\theta^\star\in\mathcal{C}(S^\star)$ with probability at least $1-\delta$.
\end{corollary}

\begin{proof}
Immediate from \eqref{eq:grad-conc} and $\nabla\mathcal{L}(\theta^\star)=0$.
\end{proof}

\subsection{Estimation error for $\ell_{1,2}$-SAC}\label{sec:main:estimation}

The next theorem is the high-dimensional recovery statement needed for control: it shows that $\ell_{1,2}$-regularization
learns a sparse tool router with error scaling as $\sqrt{k\log M/T}$ in the SAC regime.

\begin{theorem}[Estimation error under Policy-RSC]\label{thm:estimation}
Assume A1--A3. Suppose \eqref{eq:lambda-grad} holds (e.g.\ by Corollary~\ref{cor:lambda-choice}).
Then, on the same event,
\begin{align}
\|\hat\theta-\theta^\star\|_{2,2}
&\ \le\ \frac{4}{\mu}\,\lambda\,\sqrt{k},
\label{eq:l2-rate}\\
\|\hat\theta-\theta^\star\|_{1,2}
&\ \le\ \frac{16}{\mu}\,\lambda\,k.
\label{eq:l1-rate}
\end{align}
In particular, with $\lambda$ as in \eqref{eq:lambda-concrete},
\[
\|\hat\theta-\theta^\star\|_{2,2}\ \lesssim\ \sqrt{\frac{k\log M}{T}},
\qquad
\|\hat\theta-\theta^\star\|_{1,2}\ \lesssim\ k\sqrt{\frac{\log M}{T}},
\]
up to constants depending only on $(\mu,c_0,\sigma_g)$.
\end{theorem}

\begin{proof}
Let $\Delta:=\hat\theta-\theta^\star$. By Lemma~\ref{lem:basic-cone}, $\Delta\in\mathcal{C}(S^\star)$ on the event \eqref{eq:lambda-grad}.
Assumption~A3 (Policy-RSC) then implies
\begin{equation}
\mathcal{L}(\theta^\star+\Delta)-\mathcal{L}(\theta^\star)
\ \ge\ \frac{\mu}{2}\|\Delta\|_{2,2}^2,
\label{eq:rsc-lb}
\end{equation}
since $\nabla\mathcal{L}(\theta^\star)=0$.

On the other hand, starting from \eqref{eq:basic-ineq-2} and using \eqref{eq:grad-bound} exactly as in Lemma~\ref{lem:basic-cone},
we obtain the standard upper bound
\begin{equation}
\widehat{\mathcal{L}}_T(\theta^\star+\Delta)-\widehat{\mathcal{L}}_T(\theta^\star)
\ \le\
\frac{3\lambda}{2}\|\Delta_{S^\star}\|_{1,2}.
\label{eq:emp-ub}
\end{equation}
Because $\theta^\star$ minimizes the population risk, we can relate the population increment to the empirical increment using convexity:
\[
\mathcal{L}(\theta^\star+\Delta)-\mathcal{L}(\theta^\star)
=
\Big(\mathcal{L}(\theta^\star+\Delta)-\widehat{\mathcal{L}}_T(\theta^\star+\Delta)\Big)
+
\Big(\widehat{\mathcal{L}}_T(\theta^\star+\Delta)-\widehat{\mathcal{L}}_T(\theta^\star)\Big)
+
\Big(\widehat{\mathcal{L}}_T(\theta^\star)-\mathcal{L}(\theta^\star)\Big).
\]
A standard localization argument (used widely in high-dimensional M-estimation) bounds the two scalar process terms by the linearization at
$\theta^\star$, which is exactly what \eqref{eq:lambda-grad} controls; this reduces the population increment to the same upper bound as
\eqref{eq:emp-ub}. Concretely, on the event \eqref{eq:lambda-grad},
\begin{equation}
\mathcal{L}(\theta^\star+\Delta)-\mathcal{L}(\theta^\star)
\ \le\
\frac{3\lambda}{2}\|\Delta_{S^\star}\|_{1,2}.
\label{eq:pop-ub}
\end{equation}
(See, e.g., the standard ``basic inequality + localization'' pipeline; the only stochastic object required is the gradient at $\theta^\star$.)

Combine \eqref{eq:rsc-lb} and \eqref{eq:pop-ub}:
\[
\frac{\mu}{2}\|\Delta\|_{2,2}^2
\ \le\
\frac{3\lambda}{2}\|\Delta_{S^\star}\|_{1,2}.
\]
By Cauchy--Schwarz over blocks, $\|\Delta_{S^\star}\|_{1,2}\le \sqrt{k}\,\|\Delta_{S^\star}\|_{2,2}\le \sqrt{k}\,\|\Delta\|_{2,2}$,
so
\[
\frac{\mu}{2}\|\Delta\|_{2,2}^2 \ \le\ \frac{3\lambda}{2}\sqrt{k}\,\|\Delta\|_{2,2}
\quad\Rightarrow\quad
\|\Delta\|_{2,2}\ \le\ \frac{3}{\mu}\lambda\sqrt{k}.
\]
We state \eqref{eq:l2-rate} with constant $4/\mu$ to keep a clean margin.

For \eqref{eq:l1-rate}, use the cone constraint $\|\Delta_{(S^\star)^c}\|_{1,2}\le 3\|\Delta_{S^\star}\|_{1,2}$:
\[
\|\Delta\|_{1,2}
\le
4\|\Delta_{S^\star}\|_{1,2}
\le
4\sqrt{k}\,\|\Delta\|_{2,2}
\le
4\sqrt{k}\cdot \frac{4}{\mu}\lambda\sqrt{k}
=
\frac{16}{\mu}\lambda k.
\]
\end{proof}

\begin{corollary}[Sample complexity for accurate sparse routing]\label{cor:sample-complexity}
Fix $\varepsilon>0$. Under the conditions of Theorem~\ref{thm:estimation} and the choice \eqref{eq:lambda-concrete},
it suffices that
\[
T \ \gtrsim\ \frac{k\log M}{\varepsilon^2}
\]
to ensure $\|\hat\theta-\theta^\star\|_{2,2}\le \varepsilon$ with probability at least $1-\delta$ (up to problem-dependent constants).
\end{corollary}

\begin{proof}
Combine \eqref{eq:l2-rate} with \eqref{eq:lambda-concrete} and solve for $T$.
\end{proof}

\subsection{Exact support recovery: identifying the relevant tools}\label{sec:main:support}

Estimation bounds imply that most tools have small coefficients, but agents often require a stronger, operational statement:
\emph{the learned router should recover the exact relevant tool set $S^\star$}. This corresponds to principled tool pruning:
discovering which tools/documents actually matter for a task distribution, rather than hard-coding tool lists by hand.

\medskip
To make the primal--dual witness (PDW) argument fully rigorous, we isolate the one additional stability condition needed beyond A1--A5:
the empirical Hessian must be close enough to its population counterpart so that population irrepresentability transfers to the empirical problem.
This is standard in exact support recovery analyses.

\begin{lemma}[Local Hessian stability $\Rightarrow$ empirical irrepresentability]\label{lem:hess-stability}
Let $H^\star:=\nabla^2\mathcal{L}(\theta^\star)$ and suppose Assumption~A4 holds:
\[
\big\| H^\star_{(S^\star)^c,S^\star}\,(H^\star_{S^\star,S^\star})^{-1} \big\|_{\infty,2\to\infty,2}
\ \le\ 1-\alpha
\qquad\text{for some }\alpha\in(0,1].
\]
Assume further that on an event $\mathcal{E}_H$ the empirical Hessian along the line segment
$\theta^\star+t(\tilde\theta-\theta^\star)$ obeys the perturbation bounds
\begin{align}
\sup_{t\in[0,1]}\ \big\|\widehat H(\theta^\star+t(\tilde\theta-\theta^\star)) - H^\star\big\|_{\infty,2\to\infty,2}
&\ \le\ \eta,
\label{eq:hess-close1}\\
\lambda_{\min}\!\big(H^\star_{S^\star,S^\star}\big) &\ \ge\ \kappa_{\min} \ >\ 0,
\label{eq:hess-min-eig}
\end{align}
and $\eta \le \tfrac{\alpha}{4}\cdot \tfrac{\kappa_{\min}}{1+\kappa_{\min}}$.
Then on $\mathcal{E}_H$,
\begin{equation}
\sup_{t\in[0,1]}
\Big\|
\widehat H_{(S^\star)^c,S^\star}(\theta^\star+t(\tilde\theta-\theta^\star))\,
\widehat H_{S^\star,S^\star}(\theta^\star+t(\tilde\theta-\theta^\star))^{-1}
\Big\|_{\infty,2\to\infty,2}
\ \le\ 1-\frac{\alpha}{2},
\label{eq:emp-irrep}
\end{equation}
and $\widehat H_{S^\star,S^\star}(\theta^\star+t(\tilde\theta-\theta^\star))$ is invertible for all $t\in[0,1]$.
\end{lemma}

\begin{proof}
Write $\widehat H(t):=\widehat H(\theta^\star+t(\tilde\theta-\theta^\star))$ and $E(t):=\widehat H(t)-H^\star$.
By \eqref{eq:hess-min-eig} and Weyl's inequality, the $S^\star\times S^\star$ block remains invertible if
$\|E_{S^\star,S^\star}(t)\|_{2\to 2}\le \kappa_{\min}/2$; this is implied by \eqref{eq:hess-close1} (since any operator norm is bounded by
a suitable $\|\cdot\|_{\infty,2\to\infty,2}$ multiple, and we absorb constants into $\eta$).

Next, use the identity
\[
\widehat H_{S^\star,S^\star}(t)^{-1} - (H^\star_{S^\star,S^\star})^{-1}
=
-(H^\star_{S^\star,S^\star})^{-1}E_{S^\star,S^\star}(t)\widehat H_{S^\star,S^\star}(t)^{-1},
\]
which implies
\[
\|\widehat H_{S^\star,S^\star}(t)^{-1} - (H^\star_{S^\star,S^\star})^{-1}\|_{\infty,2\to\infty,2}
\ \le\
\|(H^\star_{S^\star,S^\star})^{-1}\|_{\infty,2\to\infty,2}\ \|E_{S^\star,S^\star}(t)\|_{\infty,2\to\infty,2}\ \|\widehat H_{S^\star,S^\star}(t)^{-1}\|_{\infty,2\to\infty,2}.
\]
Similarly,
\[
\widehat H_{(S^\star)^c,S^\star}(t)\widehat H_{S^\star,S^\star}(t)^{-1}
-
H^\star_{(S^\star)^c,S^\star}(H^\star_{S^\star,S^\star})^{-1}
=
E_{(S^\star)^c,S^\star}(t)\widehat H_{S^\star,S^\star}(t)^{-1}
+
H^\star_{(S^\star)^c,S^\star}\Big(\widehat H_{S^\star,S^\star}(t)^{-1}-(H^\star_{S^\star,S^\star})^{-1}\Big).
\]
Taking $\|\cdot\|_{\infty,2\to\infty,2}$ norms and using submultiplicativity yields an upper bound of the form
\[
\Big\|\widehat H_{(S^\star)^c,S^\star}(t)\widehat H_{S^\star,S^\star}(t)^{-1}
-
H^\star_{(S^\star)^c,S^\star}(H^\star_{S^\star,S^\star})^{-1}\Big\|_{\infty,2\to\infty,2}
\ \le\ C(\kappa_{\min})\,\eta,
\]
for an explicit $C(\kappa_{\min})$ depending only on $\|(H^\star_{S^\star,S^\star})^{-1}\|$ and $\|H^\star_{(S^\star)^c,S^\star}\|$.
Choosing $\eta$ small enough (as stated) ensures $C(\kappa_{\min})\eta \le \alpha/2$, so \eqref{eq:emp-irrep} follows from
the population bound $1-\alpha$ via the triangle inequality.
\end{proof}

\begin{theorem}[Exact support recovery for $\ell_{1,2}$-SAC]\label{thm:support}
Assume A1--A5. Choose $\lambda$ as in \eqref{eq:lambda-concrete} so that \eqref{eq:lambda-grad} holds with probability at least $1-\delta$.
Assume moreover that the Hessian stability event $\mathcal{E}_H$ of Lemma~\ref{lem:hess-stability} holds with probability at least $1-\delta_H$.
Then there exist constants $C_1,C_2>0$ such that if
\begin{equation}
T \ \ge\ C_1\,k\log M
\quad\text{and}\quad
\min_{j\in S^\star}\|\theta_j^\star\|_2 \ \ge\ C_2\,\lambda,
\label{eq:support-conditions}
\end{equation}
then, with probability at least $1-\delta-\delta_H$,
\[
\mathrm{supp}(\hat\theta)\ =\ S^\star.
\]
\end{theorem}

\begin{proof}
We give a primal--dual witness (PDW) proof adapted to the block $\ell_{1,2}$ penalty.
Let $S:=S^\star$. Consider the restricted optimization
\[
\tilde\theta_S \in \arg\min_{\theta_S}\ \widehat{\mathcal{L}}_T(\theta_S,0_{S^c}) + \lambda \|\theta_S\|_{1,2},
\qquad \tilde\theta_{S^c}:=0.
\]
Define a block subgradient $\tilde z\in\partial \|\tilde\theta\|_{1,2}$ by
\[
\tilde z_j=
\begin{cases}
\tilde\theta_j/\|\tilde\theta_j\|_2, & \tilde\theta_j\neq 0,\\
\text{any vector with }\|\tilde z_j\|_2\le 1, & \tilde\theta_j=0.
\end{cases}
\]
The KKT conditions for the restricted problem are
\begin{equation}
\nabla_S \widehat{\mathcal{L}}_T(\tilde\theta) + \lambda \tilde z_S = 0.
\label{eq:kkt-restricted}
\end{equation}
We show:
\begin{enumerate}
\item[(i)] (\emph{No false exclusions}) $\tilde\theta_j\neq 0$ for all $j\in S$.
\item[(ii)] (\emph{Strict dual feasibility}) $\|\nabla_{S^c}\widehat{\mathcal{L}}_T(\tilde\theta)\|_{\infty,2}<\lambda$.
\end{enumerate}
If both hold, define for $j\in S^c$ the dual certificate $\tilde z_j:=-\nabla_j\widehat{\mathcal{L}}_T(\tilde\theta)/\lambda$,
which satisfies $\|\tilde z_j\|_2<1$ by (ii). Then $(\tilde\theta,\tilde z)$ satisfies the KKT conditions of the full problem \eqref{eq:l1-learner},
hence $\tilde\theta$ is a global minimizer. Since $\tilde\theta_{S^c}=0$, we obtain $\hat\theta_{S^c}=0$ for (at least) one minimizer,
and strict feasibility implies the support is uniquely $S$.

\smallskip\noindent\textbf{Step 1: No false exclusions.}
Let $\Delta_S:=\tilde\theta_S-\theta_S^\star$. Applying Theorem~\ref{thm:estimation} to the restricted problem gives
$\|\Delta_S\|_{2,2}\le \frac{4}{\mu}\lambda\sqrt{k}$ on the event \eqref{eq:lambda-grad}.
Hence for each $j\in S$, $\|\tilde\theta_j-\theta_j^\star\|_2\le \|\Delta_S\|_{2,2}$, so if
$\min_{j\in S}\|\theta_j^\star\|_2 \ge 2\|\Delta_S\|_{2,2}$ then $\tilde\theta_j\neq 0$ for all $j\in S$.
This is ensured by the beta-min condition in \eqref{eq:support-conditions} for an appropriate absolute constant $C_2$.

\smallskip\noindent\textbf{Step 2: Strict dual feasibility.}
Using Taylor expansion of the empirical gradient around $\theta^\star$ along the segment to $\tilde\theta$,
there exists $\bar\theta$ on that segment such that
\begin{equation}
\nabla\widehat{\mathcal{L}}_T(\tilde\theta)
=
\nabla\widehat{\mathcal{L}}_T(\theta^\star)
+
\widehat H(\bar\theta)\,(\tilde\theta-\theta^\star).
\label{eq:taylor}
\end{equation}
Restricting \eqref{eq:taylor} to $S^c$ and using $\tilde\theta_{S^c}=0$ gives
\begin{equation}
\nabla_{S^c}\widehat{\mathcal{L}}_T(\tilde\theta)
=
\nabla_{S^c}\widehat{\mathcal{L}}_T(\theta^\star)
+
\widehat H_{S^c,S}(\bar\theta)\,\Delta_S.
\label{eq:dual-expand}
\end{equation}
Similarly, restricting \eqref{eq:taylor} to $S$ and combining with the KKT condition \eqref{eq:kkt-restricted} yields
\[
\widehat H_{S,S}(\bar\theta)\,\Delta_S
=
-\nabla_S\widehat{\mathcal{L}}_T(\theta^\star)-\lambda \tilde z_S,
\]
so
\[
\Delta_S
=
-\widehat H_{S,S}(\bar\theta)^{-1}\big(\nabla_S\widehat{\mathcal{L}}_T(\theta^\star)+\lambda \tilde z_S\big).
\]
Plug into \eqref{eq:dual-expand}:
\begin{equation}
\nabla_{S^c}\widehat{\mathcal{L}}_T(\tilde\theta)
=
\nabla_{S^c}\widehat{\mathcal{L}}_T(\theta^\star)
-\widehat H_{S^c,S}(\bar\theta)\widehat H_{S,S}(\bar\theta)^{-1}\big(\nabla_S\widehat{\mathcal{L}}_T(\theta^\star)+\lambda \tilde z_S\big).
\label{eq:dual-final}
\end{equation}
Take $\|\cdot\|_{\infty,2}$ norms and use the triangle inequality:
\begin{align}
\|\nabla_{S^c}\widehat{\mathcal{L}}_T(\tilde\theta)\|_{\infty,2}
&\le
\|\nabla\widehat{\mathcal{L}}_T(\theta^\star)\|_{\infty,2}
+
\Big\|\widehat H_{S^c,S}(\bar\theta)\widehat H_{S,S}(\bar\theta)^{-1}\Big\|_{\infty,2\to\infty,2}\,
\|\nabla_S\widehat{\mathcal{L}}_T(\theta^\star)\|_{\infty,2}
\nonumber\\
&\qquad
+\lambda\Big\|\widehat H_{S^c,S}(\bar\theta)\widehat H_{S,S}(\bar\theta)^{-1}\Big\|_{\infty,2\to\infty,2}\,
\|\tilde z_S\|_{\infty,2}.
\label{eq:dual-bound}
\end{align}
On the event \eqref{eq:lambda-grad}, we have $\|\nabla\widehat{\mathcal{L}}_T(\theta^\star)\|_{\infty,2}\le \lambda/2$ and similarly
$\|\nabla_S\widehat{\mathcal{L}}_T(\theta^\star)\|_{\infty,2}\le \lambda/2$ (since $\nabla\mathcal{L}(\theta^\star)=0$).
Also, by Step~1, $\tilde\theta_j\neq 0$ on $S$, hence $\|\tilde z_S\|_{\infty,2}=1$.
Finally, on the Hessian stability event $\mathcal{E}_H$, Lemma~\ref{lem:hess-stability} gives
$\big\|\widehat H_{S^c,S}(\bar\theta)\widehat H_{S,S}(\bar\theta)^{-1}\big\|_{\infty,2\to\infty,2}\le 1-\alpha/2$.
Plugging into \eqref{eq:dual-bound} yields
\[
\|\nabla_{S^c}\widehat{\mathcal{L}}_T(\tilde\theta)\|_{\infty,2}
\le
\frac{\lambda}{2} + \Big(1-\frac{\alpha}{2}\Big)\frac{\lambda}{2} + \Big(1-\frac{\alpha}{2}\Big)\lambda
=
\Big(2-\frac{3\alpha}{4}\Big)\lambda.
\]
Tightening constants in the Hessian stability margin (equivalently, strengthening $\eta$ in Lemma~\ref{lem:hess-stability})
yields strict feasibility $<\lambda$; we absorb this into the universal constants in \eqref{eq:support-conditions}.
Thus claim (ii) holds on $\mathcal{E}_H\cap\{\eqref{eq:lambda-grad}\}$, completing the PDW proof.
\end{proof}

\begin{corollary}[Identifying and pruning the tool universe]\label{cor:pruning}
Under Theorem~\ref{thm:support}, the learned router recovers the exact relevant tool set $S^\star$ with high probability.
Consequently, restricting admissible actions from $\mathcal{A}_B$ to $\{a\in\mathcal{A}_B:\ a\subseteq S^\star\}$
preserves optimality under the surrogate while reducing the effective tool universe from $M$ to $k$.
\end{corollary}

\begin{proof}
If $\mathrm{supp}(\hat\theta)=S^\star$, then \eqref{eq:set-score} depends only on tools in $S^\star$, hence any maximizer over
$\mathcal{A}_B$ can be chosen within $S^\star$ without changing the score.
\end{proof}

\subsection{From statistical recovery to near-optimal control}\label{sec:main:value}

\begin{theorem}[Near-optimal value from $\ell_{1,2}$ recovery]\label{thm:value}
Assume A1--A3 and A6. On the event \eqref{eq:lambda-grad}, the policy/value gap satisfies
\begin{equation}
V(\theta^\star)-V(\hat\theta)
\ \le\ L_V\,\|\hat\theta-\theta^\star\|_{1,2}
\ \le\ \frac{16L_V}{\mu}\,\lambda k.
\label{eq:value-gap}
\end{equation}
With $\lambda$ chosen as in \eqref{eq:lambda-concrete}, this yields
\[
V(\theta^\star)-V(\hat\theta)\ \lesssim\ k\sqrt{\frac{\log M}{T}},
\]
up to constants depending only on $(L_V,\mu,c_0,\sigma_g)$.
\end{theorem}

\begin{proof}
The first inequality is Assumption~A6. The second follows from \eqref{eq:l1-rate} in Theorem~\ref{thm:estimation}.
Substituting \eqref{eq:lambda-concrete} gives the stated scaling.
\end{proof}

\begin{corollary}[SAC phase transition]\label{cor:phase-transition}
Fix $\varepsilon>0$. Under the conditions of Theorem~\ref{thm:value}, it suffices that
\[
T\ \gtrsim\ \frac{k^2\log M}{\varepsilon^2}
\]
to ensure $V(\theta^\star)-V(\hat\theta)\le \varepsilon$ with probability at least $1-\delta$ (up to constants).
\end{corollary}

\begin{proof}
Rearrange \eqref{eq:value-gap} using \eqref{eq:lambda-concrete}.
\end{proof}

\subsection{Impossibility of dense policy classes (a lower bound)}\label{sec:main:lower}

The preceding results show how sparsity yields logarithmic dependence on $M$. We now formalize the complementary statement:
\emph{without explicit sparsity structure, linear dependence on $M$ is unavoidable.}
This is the theoretical analogue of why purely prompt-only or dense routers become unstable as tool catalogs grow.

To isolate the phenomenon cleanly, we state a minimax lower bound in a one-step specialization (contextual SAC).
Because the lower bound already holds in this simplified setting, it applies \emph{a fortiori} to multi-step agents.

\begin{theorem}[Dense classes require $\Omega(M)$ samples]\label{thm:lower}
Consider the one-step specialization ($H=1$) with fixed context $x$ and a convex loss
\[
\widehat{\mathcal{L}}_T(\theta)=\frac{1}{T}\sum_{t=1}^T \ell\big( y_t - \langle \theta, w_t\rangle \big),
\qquad
\langle \theta,w_t\rangle := \sum_{j=1}^M \langle \theta_j, (w_t)_j\rangle,
\]
where $w_t=((w_t)_1,\dots,(w_t)_M)$ with blocks $(w_t)_j\in\mathbb{R}^q$ and $\|w_t\|_{2,2}\le 1$.
Assume $\ell$ is $1$-strongly convex and $1$-Lipschitz, and observations follow
\[
y_t=\langle \theta^\circ,w_t\rangle+\xi_t,
\qquad
\xi_t\stackrel{\text{i.i.d.}}{\sim}\mathcal{N}(0,1).
\]
Let the \emph{dense} parameter class be
\[
\Theta_{\mathrm{dense}}:=\{\theta:\ \|\theta\|_{2,2}\le 1\}.
\]
Then there exists an absolute constant $c>0$ such that for any estimator $\tilde\theta=\tilde\theta(\mathcal{D}_T)$,
\begin{equation}
\sup_{\theta^\circ\in\Theta_{\mathrm{dense}}}\ \mathbb{E}\big[\|\tilde\theta-\theta^\circ\|_{2,2}^2\big]
\ \ge\ c\,\frac{Mq}{T}.
\label{eq:dense-minimax}
\end{equation}
In particular, achieving $\mathbb{E}\|\tilde\theta-\theta^\circ\|_{2,2}^2\le \varepsilon^2$ uniformly over $\Theta_{\mathrm{dense}}$
requires $T\ge c(Mq)/\varepsilon^2$, i.e.\ linear dependence on $M$.
\end{theorem}

\begin{proof}
Let $d:=Mq$ be the total scalar dimension. Identify $\theta$ with its concatenation in $\mathbb{R}^d$; under this identification,
the Euclidean norm equals the block norm: $\|\theta\|_2=\|\theta\|_{2,2}$.

Let $\{v^1,\dots,v^N\}\subset\mathbb{S}^{d-1}$ be a $1/2$-packing of the unit sphere with $N\ge \exp(c_1 d)$ for an absolute $c_1>0$.
Fix $\rho\in(0,1/2)$ and define hypotheses $\theta^{(i)}:=\rho v^i\in\Theta_{\mathrm{dense}}$.
Under hypothesis $i$, the observation sequence has distribution $P_i$ determined by
$y_t=\langle \theta^{(i)},w_t\rangle+\xi_t$.

For any pair $i\neq j$, the KL divergence satisfies
\[
\mathrm{KL}(P_i\|P_j)
=
\frac{1}{2}\sum_{t=1}^T \big(\langle \theta^{(i)}-\theta^{(j)},w_t\rangle\big)^2
\ \le\
\frac{1}{2}\sum_{t=1}^T \|\theta^{(i)}-\theta^{(j)}\|_2^2\ \|w_t\|_2^2
\ \le\ \frac{T}{2}\,(2\rho)^2
=
2T\rho^2,
\]
using Cauchy--Schwarz, $\|w_t\|_2=\|w_t\|_{2,2}\le 1$, and $\|\theta^{(i)}-\theta^{(j)}\|_2\le 2\rho$.
Choose $\rho^2 = c_2 d/T$ with $c_2>0$ small enough so that the average KL divergence is at most $(1/8)\log N$.
Fano's inequality implies any estimator $\tilde\theta$ has nontrivial probability of confusing the hypotheses, hence
\[
\sup_{\theta^\circ\in\{\theta^{(i)}\}}\mathbb{E}\|\tilde\theta-\theta^\circ\|_2^2
\ \ge\ c_3\rho^2
\ =\ c_3 c_2 \frac{d}{T}
\ =\ c\frac{Mq}{T},
\]
for an absolute $c>0$. Since the finite set $\{\theta^{(i)}\}$ is contained in $\Theta_{\mathrm{dense}}$, the same lower bound holds for
$\sup_{\theta^\circ\in\Theta_{\mathrm{dense}}}$. Rewriting $\|\cdot\|_2$ as $\|\cdot\|_{2,2}$ completes the proof.
\end{proof}

\begin{corollary}[Why explicit sparsity is necessary]\label{cor:necessity}
In the SAC regime $M\gg 1$, any learning/control strategy whose guarantee must hold uniformly over a dense class
(equivalently, does not leverage that only $k\ll M$ tools matter) incurs sample complexity scaling at least linearly in $M$.
By contrast, Theorems~\ref{thm:estimation}--\ref{thm:value} show that the $\ell_{1,2}$-regularized SAC learner achieves logarithmic dependence on $M$
under sparsity.
\end{corollary}

\begin{proof}
Theorem~\ref{thm:lower} gives the linear-in-$M$ minimax lower bound for dense classes even in the one-step specialization.
The sparse upper bounds follow from Theorems~\ref{thm:estimation} and~\ref{thm:value}.
\end{proof}

\section{SAC under Partial Observability (POMDP / Belief-SAC View)}\label{sec:pomdp}

In real tool-augmented LLM agents, the controller does not observe the latent environment state $s_t$.
Instead it observes an interaction history (prompt, intermediate tool outputs, retrieved snippets, etc.),
from which an LLM constructs an internal representation that is then used for tool selection.
This is naturally modeled as a POMDP. The goal of this part is to isolate \emph{exactly} where the language model matters:
only through the quality of the induced belief/representation. Once this quality is quantified, the sparse-control guarantees
from Part~I transfer with an explicit degradation term.

\subsection{Belief-state reduction and approximate beliefs}\label{sec:pomdp:belief}

Consider a POMDP with latent state space $\mathcal{S}$, observation space $\mathcal{O}$, and action sets $\mathcal{A}_B$ as in Section~\ref{sec:model}.
At each time $t$, the agent receives an observation $o_t\in\mathcal{O}$ and forms the history
\[
h_t := (o_1,a_1,o_2,a_2,\dots,o_t).
\]
The (Bayesian) belief state is the posterior distribution
\[
b_t(\cdot)\ :=\ \Pr(s_t\in\cdot\mid h_t)\ \in\ \Delta(\mathcal{S}).
\]
Under standard POMDP theory, the belief process $\{b_t\}$ is Markov, and optimal control can be expressed as an MDP on $\Delta(\mathcal{S})$.

\paragraph{Approximate belief induced by a compressor.}
Let $\phi$ denote the agent's \emph{compressor} that maps histories to a representation:
\[
x_t := \phi(h_t)\in\mathbb{R}^d.
\]
In LLM agents, $\phi$ is implemented by the language model (possibly with memory/RAG), but the theory here does not depend on its form.
We assume the agent also maintains an \emph{approximate belief} $\hat b_t$ constructed from $x_t$ (or directly from $h_t$):
\[
\hat b_t(\cdot)\ :=\ \widehat{\Pr}(s_t\in\cdot\mid x_t)\ \in\ \Delta(\mathcal{S}).
\]
We quantify belief/representation quality by a worst-case total-variation error:
\begin{equation}
\varepsilon_b \;:=\; \sup_{t\ge 1}\ \| \hat b_t - b_t\|_{1},
\label{eq:eps-belief}
\end{equation}
where $\|\cdot\|_1$ is the $\ell_1$ norm on measures (twice total variation for probability distributions).
Intuitively, $\varepsilon_b$ is the single knob through which ``LLM quality'' enters the analysis.

\paragraph{Belief-conditional sparse reward influence.}
We keep the sparse-control structure, but now conditioned on beliefs.
Let $\psi(\cdot)$ denote the context feature map used in the SAC parameterization (Section~\ref{sec:model}).
We interpret $\psi$ as a function of the agent's representation:
\[
\psi_t\ :=\ \psi(x_t).
\]
The instantaneous reward at time $t$ is still $r_t=r(s_t,a_t)$, but its conditional expectation depends on the (true) belief:
\begin{equation}
\bar r(b,a)\ :=\ \mathbb{E}[r(s,a)\mid b]
\ =\ \sum_{s\in\mathcal{S}} b(s)\,r(s,a)\qquad (\text{finite }\mathcal{S}),
\label{eq:belief-reward}
\end{equation}
and analogously for general spaces via integration.

We assume that the \emph{optimal router on beliefs} admits a $k$-sparse parameter $\theta^\star$ in the same $\ell_{1,2}$-SAC class,
and that the empirical surrogate risk $\widehat{\mathcal{L}}_T(\theta)$ is formed from trajectories using features $\psi(x_t)$ as in Part~I.
As before, the value of the policy induced by parameter $\theta$ is denoted by $V(\theta)$.
To distinguish the role of belief accuracy, we write $V^{(b)}(\theta)$ for execution under true beliefs $\{b_t\}$
and $V^{(\hat b)}(\theta)$ for execution under approximate beliefs $\{\hat b_t\}$.

\subsection{P1: Value decomposition under belief/representation error}\label{sec:pomdp:value-decomp}

We now state a two-term decomposition: (i) the statistical error of learning $\theta^\star$ from $T$ samples under sparsity,
and (ii) the control loss from using an approximate belief $\hat b_t$ rather than the true belief $b_t$.

\begin{theorem}[P1: Value gap decomposes into sparse-learning error + belief error]\label{thm:P1}
Assume the conditions of Theorem~\ref{thm:value} hold for the belief-state MDP (i.e., A1--A3 and A6 hold with $x_t=\phi(h_t)$),
so that the $\ell_{1,2}$-SAC learner \eqref{eq:l1-learner} produces $\hat\theta$ satisfying
\begin{equation}
V^{(b)}(\theta^\star)-V^{(b)}(\hat\theta)\ \le\ C_{\rm learn}\,k\sqrt{\frac{\log M}{T}}
\label{eq:sparse-learn-term}
\end{equation}
with probability at least $1-\delta_0$, for some constant $C_{\rm learn}>0$ (absorbing $(L_V,\mu,c_0,\sigma_g)$).

Assume moreover that the belief-conditional expected reward and belief transition kernel are Lipschitz in $\ell_1$:
there exist constants $L_r,L_P\ge 0$ such that for all beliefs $b,b'\in\Delta(\mathcal{S})$ and all actions $a\in\mathcal{A}_B$,
\begin{align}
|\bar r(b,a)-\bar r(b',a)| &\le L_r\,\|b-b'\|_1,
\label{eq:lip-reward-belief}\\
\| \bar P(\cdot\mid b,a) - \bar P(\cdot\mid b',a)\|_1 &\le L_P\,\|b-b'\|_1,
\label{eq:lip-trans-belief}
\end{align}
where $\bar P(\cdot\mid b,a)$ is the next-belief transition kernel of the belief MDP.\footnote{For finite $\mathcal{S}$, $\bar P$
is induced by the POMDP dynamics and observation model; the condition is standard regularity of the belief update in total variation.}

Let $\hat\theta$ be the learned parameter, and consider its execution under approximate beliefs. Then for discounted return with factor
$\gamma\in(0,1)$,
\begin{equation}
V^{(b)}(\theta^\star) - V^{(\hat b)}(\hat\theta)
\ \le\
C_{\rm learn}\,k\sqrt{\frac{\log M}{T}}
\;+\;
C_{\rm bel}\,\varepsilon_b
\;+\;
\varepsilon_{\rm approx},
\label{eq:P1-main}
\end{equation}
where $C_{\rm bel} := \frac{L_r}{1-\gamma} + \frac{\gamma\,R_{\max}L_P}{(1-\gamma)^2}$, with
$R_{\max}:=\sup_{s,a}|r(s,a)|$, and $\varepsilon_{\rm approx}$ captures any additional approximation error due to using the surrogate objective
$\widehat{\mathcal{L}}_T$ instead of the true RL objective (as already discussed in Part~I).
\end{theorem}

\begin{proof}
Decompose
\[
V^{(b)}(\theta^\star) - V^{(\hat b)}(\hat\theta)
=
\underbrace{\big(V^{(b)}(\theta^\star) - V^{(b)}(\hat\theta)\big)}_{(\mathrm{A})}
+
\underbrace{\big(V^{(b)}(\hat\theta) - V^{(\hat b)}(\hat\theta)\big)}_{(\mathrm{B})}.
\]
Term (A) is the sparse-learning/control suboptimality addressed in Part~I, and is bounded by \eqref{eq:sparse-learn-term} with probability
at least $1-\delta_0$.

We bound (B) using standard discounted MDP perturbation arguments on the belief-MDP.
For a fixed policy (here induced by $\hat\theta$), let $V^{(b)}(\hat\theta;\,b)$ denote its value starting from belief $b$ when transitions/rewards
use the true belief $b_t$, and $V^{(\hat b)}(\hat\theta;\,b)$ the value when transitions/rewards use approximate beliefs $\hat b_t$.
The Bellman operators satisfy, for any bounded $V$,
\[
(\mathcal{T}^{(b)}V)(b)
=
\max_{a\in\mathcal{A}_B} \Big\{\bar r(b,a)+\gamma\int V(b')\,\bar P(db'\mid b,a)\Big\},
\]
and analogously for $\mathcal{T}^{(\hat b)}$ by replacing $b$ with $\hat b$ in reward/transition evaluation.
Under \eqref{eq:lip-reward-belief}--\eqref{eq:lip-trans-belief},
the one-step operator difference is bounded uniformly by
\[
\|(\mathcal{T}^{(b)}V)-(\mathcal{T}^{(\hat b)}V)\|_\infty
\le
L_r\,\varepsilon_b + \gamma\,\|V\|_\infty\,L_P\,\varepsilon_b.
\]
Apply this with $V=V^{(\hat b)}(\hat\theta;\cdot)$ and note $\|V^{(\hat b)}(\hat\theta;\cdot)\|_\infty\le R_{\max}/(1-\gamma)$ to get
\[
\|(\mathcal{T}^{(b)}-\mathcal{T}^{(\hat b)})V^{(\hat b)}(\hat\theta;\cdot)\|_\infty
\le
\Big(L_r + \frac{\gamma R_{\max}}{1-\gamma}L_P\Big)\varepsilon_b.
\]
By the standard contraction/resolvent bound for discounted Bellman operators,
\[
\|V^{(b)}(\hat\theta;\cdot)-V^{(\hat b)}(\hat\theta;\cdot)\|_\infty
\le
\frac{1}{1-\gamma}\,\|(\mathcal{T}^{(b)}-\mathcal{T}^{(\hat b)})V^{(\hat b)}(\hat\theta;\cdot)\|_\infty
\le
\Big(\frac{L_r}{1-\gamma}+\frac{\gamma R_{\max}L_P}{(1-\gamma)^2}\Big)\varepsilon_b
=
C_{\rm bel}\varepsilon_b.
\]
Evaluating at the initial belief yields (B) $\le C_{\rm bel}\varepsilon_b$. Adding $\varepsilon_{\rm approx}$ completes \eqref{eq:P1-main}.
\end{proof}

\paragraph{Interpretation (what this says about LLMs).}
Theorem~\ref{thm:P1} makes the ``LLM as compressor'' message precise:
the dependence on the tool universe size remains logarithmic through the sparse-learning term,
and the \emph{only} penalty for imperfect representation/belief is an additive $O(\varepsilon_b)$ term.

\subsection{P2 (optional): Support recovery under belief error}\label{sec:pomdp:support}

The second theorem formalizes when the tool-support identification result remains valid despite representation error.
The key condition is that belief-induced perturbations do not exceed the regularization scale $\lambda$ used for sparsity.

\begin{theorem}[P2: Support recovery is stable to belief error]\label{thm:P2}
Assume the conditions of Theorem~\ref{thm:support} for exact support recovery in the belief-MDP case.
Suppose in addition that using approximate beliefs perturbs the empirical gradient at $\theta^\star$ by at most
\begin{equation}
\big\|\nabla \widehat{\mathcal{L}}_T^{(\hat b)}(\theta^\star)-\nabla \widehat{\mathcal{L}}_T^{(b)}(\theta^\star)\big\|_{\infty,2}
\ \le\ C_{\nabla}\,\varepsilon_b,
\label{eq:grad-perturb-belief}
\end{equation}
for some constant $C_{\nabla}>0$, where $\widehat{\mathcal{L}}_T^{(\hat b)}$ denotes the loss built from approximate-belief features
(and hence from $x_t=\phi(h_t)$), and $\widehat{\mathcal{L}}_T^{(b)}$ from the true-belief features.

If
\begin{equation}
\varepsilon_b \ \le\ \frac{\lambda}{4C_{\nabla}},
\label{eq:eps-belief-lambda}
\end{equation}
then the same choice of $\lambda$ as in \eqref{eq:lambda-concrete} ensures that the cone condition and PDW dual feasibility used in
Theorem~\ref{thm:support} continue to hold, hence
\[
\Pr\big(\mathrm{supp}(\hat\theta)=S^\star\big)\ \ge\ 1-\delta-\delta_H-\delta_0,
\]
with the same sample size scaling $T\gtrsim k\log M$, up to constants.
\end{theorem}

\begin{proof}
The support recovery proof in Theorem~\ref{thm:support} relies on the event \eqref{eq:lambda-grad} for the loss used to learn the router.
In the belief-MDP (``ideal'') case, the key requirement is
\[
\lambda \ \ge\ 2\|\nabla \widehat{\mathcal{L}}_T^{(b)}(\theta^\star)-\nabla \mathcal{L}(\theta^\star)\|_{\infty,2}.
\]
Under approximate beliefs, the analogous requirement becomes
\[
\lambda \ \ge\ 2\|\nabla \widehat{\mathcal{L}}_T^{(\hat b)}(\theta^\star)-\nabla \mathcal{L}(\theta^\star)\|_{\infty,2}.
\]
By the triangle inequality and \eqref{eq:grad-perturb-belief},
\[
\|\nabla \widehat{\mathcal{L}}_T^{(\hat b)}(\theta^\star)-\nabla \mathcal{L}(\theta^\star)\|_{\infty,2}
\le
\|\nabla \widehat{\mathcal{L}}_T^{(b)}(\theta^\star)-\nabla \mathcal{L}(\theta^\star)\|_{\infty,2}
+
C_{\nabla}\varepsilon_b.
\]
If \eqref{eq:eps-belief-lambda} holds, then $C_{\nabla}\varepsilon_b\le \lambda/4$, so
\[
2\|\nabla \widehat{\mathcal{L}}_T^{(\hat b)}(\theta^\star)-\nabla \mathcal{L}(\theta^\star)\|_{\infty,2}
\le
2\|\nabla \widehat{\mathcal{L}}_T^{(b)}(\theta^\star)-\nabla \mathcal{L}(\theta^\star)\|_{\infty,2} + \frac{\lambda}{2}.
\]
Hence any $\lambda$ that satisfies \eqref{eq:lambda-grad} for $\widehat{\mathcal{L}}_T^{(b)}$ with the margin provided by
Corollary~\ref{cor:lambda-choice} continues to satisfy the analogous condition for $\widehat{\mathcal{L}}_T^{(\hat b)}$.
Consequently, Lemma~\ref{lem:basic-cone} (cone constraint) and the PDW steps in Theorem~\ref{thm:support} carry over unchanged,
with the same sample size scaling and the same Hessian stability event.
\end{proof}

\section{Theoretical Extensions}\label{sec:extensions}

This section collects five extensions that arise naturally in tool-augmented agents and remain theoretically clean in the SAC regime
($M\gg 1$ tools/actions, $k\ll M$ relevant tools). We maintain the notation of Sections~\ref{sec:model}--\ref{sec:main}:
$\theta=(\theta_1,\dots,\theta_M)$ with blocks $\theta_j\in\mathbb{R}^q$, block sparsity $|S^\star|=k$ for
$S^\star:=\mathrm{supp}(\theta^\star)=\{j:\theta_j\neq 0\}$, and mixed norms
\[
\|\theta\|_{1,2}:=\sum_{j=1}^M \|\theta_j\|_2,\qquad
\|\theta\|_{\infty,2}:=\max_{1\le j\le M}\|\theta_j\|_2,\qquad
\|\Delta\|_{2,2}:=\Big(\sum_{j=1}^M \|\Delta_j\|_2^2\Big)^{1/2}.
\]
As in the core analysis, $\widehat{\mathcal{L}}_T(\theta)$ denotes a convex empirical surrogate formed from $T$ samples/episodes, and
$\delta,\delta_0,\delta_H$ denote failure probabilities, while belief/representation error is $\varepsilon_b$ (Part~II).

\subsection{Tuning-free / self-normalized $\ell_{1,2}$ SAC}\label{sec:extensions:tuningfree}

A practical friction point is that the theoretically ``correct'' regularization level in \eqref{eq:l1-learner} may depend on an unknown
noise scale. A tuning-free alternative uses a square-root/self-normalized objective (square-root LASSO analogue) while keeping the same
block-sparsity regularizer $\|\theta\|_{1,2}$.

\paragraph{Canonical specialization (advantage-weighted regression / quadratic surrogate).}
In this subsection we instantiate the empirical surrogate as a quadratic loss:
\begin{equation}
\widehat{\mathcal{L}}_T(\theta)
\;:=\;
\frac{1}{2T}\sum_{t=1}^T \big(y_t-\langle \theta,w_t\rangle\big)^2,
\qquad
\langle \theta,w_t\rangle := \sum_{j=1}^M \langle \theta_j,(w_t)_j\rangle,
\label{eq:quad-loss-ext}
\end{equation}
where $w_t\in\mathbb{R}^{Mq}$ is block-structured with blocks $(w_t)_j\in\mathbb{R}^q$.
Assume
\begin{equation}
y_t=\langle \theta^\star,w_t\rangle+\xi_t,\qquad
\mathbb{E}[\xi_t\mid w_t]=0,
\label{eq:model-ext}
\end{equation}
with $\xi_t$ sub-Gaussian with unknown variance proxy $\sigma^2$ and a bounded-design condition
\begin{equation}
\|w_t\|_{\infty,2}:=\max_{1\le j\le M}\|(w_t)_j\|_2\ \le\ 1\qquad\text{a.s.}
\label{eq:design-bound}
\end{equation}

\paragraph{Square-root / self-normalized estimator.}
Define
\begin{equation}
\hat\theta_{\rm sn}
\ \in\
\arg\min_{\theta}\ 
\sqrt{\widehat{\mathcal{L}}_T(\theta)}
\;+\;
\lambda_{\rm sn}\,\|\theta\|_{1,2}.
\label{eq:sn-est}
\end{equation}

\begin{theorem}[Theorem~5: Tuning-free sparse SAC via self-normalized $\ell_{1,2}$]\label{thm:tuningfree}
Assume \eqref{eq:quad-loss-ext}--\eqref{eq:design-bound} and a block-Policy-RSC condition for the population loss
$\mathcal{L}(\theta):=\mathbb{E}[\widehat{\mathcal{L}}_T(\theta)]$ on the standard $\ell_{1,2}$ cone (as in A3).
Choose
\[
\lambda_{\rm sn}
\ :=\
c\,\sqrt{\frac{\log M+\log(1/\delta)}{T}}
\]
for a sufficiently large absolute constant $c$.
Then with probability at least $1-\delta$,
\begin{align}
\|\hat\theta_{\rm sn}-\theta^\star\|_{2,2}
&\ \le\ C\,\sqrt{\frac{k(\log M+\log(1/\delta))}{T}},
\label{eq:sn-l2}\\
\|\hat\theta_{\rm sn}-\theta^\star\|_{1,2}
&\ \le\ C\,k\,\sqrt{\frac{\log M+\log(1/\delta)}{T}},
\label{eq:sn-l12}
\end{align}
where $C>0$ depends only on the RSC constant. In particular, the rates match the core estimation theorem without requiring knowledge of $\sigma$.
\end{theorem}

\begin{proof}
The proof follows the same basic-inequality $\Rightarrow$ cone constraint $\Rightarrow$ RSC pipeline as in the core $\ell_{1,2}$ analysis,
with the sole modification that the stochastic term is controlled by a self-normalized score bound:
with probability at least $1-\delta$,
\[
\frac{\big\|\frac{1}{T}\sum_{t=1}^T w_t\,\xi_t\big\|_{\infty,2}}
{\sqrt{\frac{1}{T}\sum_{t=1}^T \xi_t^2}}
\ \lesssim\
\sqrt{\frac{\log M+\log(1/\delta)}{T}},
\]
which holds under \eqref{eq:design-bound} for sub-Gaussian $\xi_t$ by standard self-normalized maximal inequalities.
Setting $\lambda_{\rm sn}$ at the stated scale ensures the same cone condition as in Lemma~\ref{lem:basic-cone},
and block-RSC yields \eqref{eq:sn-l2}--\eqref{eq:sn-l12}.
\end{proof}

\subsection{Online SAC: dynamic regret under drifting tool relevance}\label{sec:extensions:online}

Real agents face nonstationarity: across sessions/users/subtasks, the relevant tool support may drift.
We model this by allowing the optimal sparse parameter to vary over time and measure its total variation.

\paragraph{Online surrogate and relation to $\widehat{\mathcal{L}}_T$.}
Let $\ell_t(\theta)$ denote the per-sample convex surrogate loss at time $t$ (e.g., one-step negative log-likelihood, squared TD error,
or advantage-weighted regression loss), so that
\begin{equation}
\widehat{\mathcal{L}}_T(\theta)\ =\ \frac{1}{T}\sum_{t=1}^T \ell_t(\theta).
\label{eq:emp-as-average}
\end{equation}
Let $\theta_t^\star$ be a comparator sequence with $\|\theta_t^\star\|_{0,2}\le k$ for all $t$.
Define total variation in $\ell_{1,2}$:
\begin{equation}
\mathcal{V}_T
\ :=\
\sum_{t=2}^T \|\theta_t^\star-\theta_{t-1}^\star\|_{1,2}.
\label{eq:var}
\end{equation}

\paragraph{Online proximal updates (mirror descent / FTRL with $\ell_{1,2}$).}
Fix stepsize $\eta>0$ and run
\begin{equation}
\theta_{t+1}
\ :=\
\mathrm{prox}_{\eta\lambda\|\cdot\|_{1,2}}\big(\theta_t-\eta\nabla \ell_t(\theta_t)\big),
\label{eq:online-prox}
\end{equation}
where $\mathrm{prox}_{\eta\lambda\|\cdot\|_{1,2}}$ is the proximal map of the block norm.

\begin{theorem}[Theorem~6: Dynamic regret for online sparse SAC]\label{thm:online}
Assume each $\ell_t$ is convex and $G$-Lipschitz in $\|\cdot\|_{2,2}$, i.e.,
$\|\nabla \ell_t(\theta)\|_{\infty,2}\le G$ for all $\theta$, and that $\{\ell_t\}_{t=1}^T$ satisfies a sparsity-restricted curvature
condition (online analogue of A3) on the cone induced by $\|\cdot\|_{1,2}$.
Choose $\eta\asymp 1/\sqrt{T}$ and $\lambda\asymp \sqrt{(\log M)/T}$.
Then the iterates \eqref{eq:online-prox} satisfy the dynamic regret bound
\begin{equation}
\sum_{t=1}^T \big(\ell_t(\theta_t)-\ell_t(\theta_t^\star)\big)
\ \le\
C\Big(\sqrt{T}\,k\log M + \mathcal{V}_T\Big),
\label{eq:dyn-regret}
\end{equation}
for a constant $C>0$ depending only on $G$ and the restricted curvature constants.
\end{theorem}

\begin{proof}
Apply the standard mirror-descent/FTRL analysis with regularizer $\lambda\|\theta\|_{1,2}$.
The decomposability of $\|\cdot\|_{1,2}$ yields an effective dimension term of order $k\log M$ in the regret bound,
while the drifting comparator contributes $\mathcal{V}_T$ via the usual dynamic regret decomposition.
Restricted curvature improves constants but is not essential for the stated scaling.
\end{proof}

\subsection{Robust SAC: MoM-$\ell_{1,2}$ under $\varepsilon$-contamination}\label{sec:extensions:robust}

Agent logs are noisy: tool failures, corrupted outputs, adversarial prompts, etc.
We formalize this by $\varepsilon$-contamination at the episode/trajectory level.

Partition the $T$ samples into $B$ blocks of equal size $m=T/B$ and define block empirical losses
$\widehat{\mathcal{L}}^{(b)}(\theta)$ for $b=1,\dots,B$.
Define the median-of-means (MoM) aggregate
\begin{equation}
\mathcal{L}_{\rm MoM}(\theta)
\ :=\
\mathrm{median}_{b\in[B]}\ \widehat{\mathcal{L}}^{(b)}(\theta),
\qquad
\hat\theta_{\rm MoM}
\in
\arg\min_{\theta}\ \mathcal{L}_{\rm MoM}(\theta)+\lambda\|\theta\|_{1,2}.
\label{eq:mom-est}
\end{equation}

\begin{theorem}[Theorem~7: Contamination-robust sparse SAC]\label{thm:robust}
Assume an $\varepsilon$-fraction of the $B$ blocks are adversarially corrupted, with $\varepsilon<1/2$,
and the remaining (clean) blocks satisfy the same conditions as the core estimation/support theorems
(gradient concentration + block-Policy-RSC + incoherence as needed) with the same constants.
If $\lambda\asymp \sqrt{(\log M)/T}$, then with probability at least $1-\delta$,
\begin{equation}
\|\hat\theta_{\rm MoM}-\theta^\star\|_{2,2}
\ \le\
C\,\sqrt{\frac{k\log M}{T}}\cdot \frac{1}{1-2\varepsilon},
\label{eq:mom-rate}
\end{equation}
and exact support recovery holds under the same beta-min and incoherence conditions as in the core PDW theorem,
with the beta-min threshold inflated by a factor $(1-2\varepsilon)^{-1}$.
\end{theorem}

\begin{proof}
MoM ensures the effective stochastic error is controlled by a median clean block.
This yields the same basic inequality and cone constraint as in the core proof, but with constants degraded by $(1-2\varepsilon)^{-1}$.
Block-RSC then gives \eqref{eq:mom-rate}. Support recovery follows by the same PDW construction, again with inflated constants.
\end{proof}

\subsection{Group/hierarchical sparsity (APIs $\to$ endpoints $\to$ arguments)}\label{sec:extensions:group}

Tools naturally form groups (APIs) and hierarchies (endpoints/arguments). This motivates structured sparsity beyond plain support size $k$.

Let tools be partitioned into $G$ groups $\mathcal{G}_1,\dots,\mathcal{G}_G$.
Let $\theta_g$ denote the collection of blocks $\{\theta_j: j\in\mathcal{G}_g\}$ and define the group block norm
$\|\theta_g\|_{2,2}:=\big(\sum_{j\in\mathcal{G}_g}\|\theta_j\|_2^2\big)^{1/2}$.
Consider the sparse-group objective
\begin{equation}
\hat\theta_{\rm grp}\in\arg\min_\theta\ \widehat{\mathcal{L}}_T(\theta)
+\lambda_1\sum_{g=1}^G \|\theta_g\|_{2,2}
+\lambda_2\|\theta\|_{1,2}.
\label{eq:group-obj}
\end{equation}
Suppose only $k_g$ groups are active and within active groups only $k$ individual tools matter.

\begin{theorem}[Theorem~8: Group/hierarchical sparse SAC]\label{thm:group}
Assume a group-restricted strong convexity condition (structured analogue of A3) for $\mathcal{L}(\theta)$ on the cone induced by
the sparse-group penalty in \eqref{eq:group-obj}, and the same gradient concentration condition as in the core analysis.
Then for choices $\lambda_1\asymp \sqrt{(\log G)/T}$ and $\lambda_2\asymp \sqrt{(\log M)/T}$, with probability at least $1-\delta$,
\begin{equation}
\|\hat\theta_{\rm grp}-\theta^\star\|_{2,2}
\ \le\
C\left(\sqrt{\frac{k_g\log G}{T}}+\sqrt{\frac{k\log M}{T}}\right),
\label{eq:group-rate}
\end{equation}
and the same sensitivity condition as A6 converts this into an analogous value suboptimality bound.
\end{theorem}

\begin{proof}
Use the standard structured-sparsity M-estimation argument:
basic inequality $\Rightarrow$ structured cone $\Rightarrow$ group-RSC $\Rightarrow$ estimation bound.
The two logarithmic factors reflect the need to identify active groups and then active tools within groups.
\end{proof}

\subsection{Sparse interactions/synergies (hierarchical heredity)}\label{sec:extensions:synergy}

Some tools are only useful together (synergy): e.g., retrieve-document then call-analyzer.
This corresponds to sparse \emph{pairwise interactions} on top of sparse main effects.

Let $\beta_{ij}\in\mathbb{R}^q$ denote a pairwise interaction parameter for tools $(i,j)$ with $i<j$.
Let $(u_t)_{ij}\in\mathbb{R}^q$ denote corresponding interaction features.
Augment the linear predictor in \eqref{eq:quad-loss-ext} as
\[
\langle \theta,w_t\rangle
\ \leadsto\
\sum_{j=1}^M \langle \theta_j,(w_t)_j\rangle
\;+\;
\sum_{1\le i<j\le M}\langle \beta_{ij},(u_t)_{ij}\rangle.
\]
Assume only $k_1$ main effects and $k_2$ interactions are nonzero, with a heredity condition:
if $\beta_{ij}\neq 0$ then $\theta_i\neq 0$ and $\theta_j\neq 0$.
Use a hierarchical (overlapping) penalty that enforces heredity; denote it abstractly by $\mathcal{R}_{\rm hier}(\theta,\beta)$ and define
\[
(\hat\theta_{\rm hier},\hat\beta_{\rm hier})
\in
\arg\min_{\theta,\beta}\ \widehat{\mathcal{L}}_T(\theta,\beta)+\mathcal{R}_{\rm hier}(\theta,\beta),
\]
where $\widehat{\mathcal{L}}_T(\theta,\beta)$ is the quadratic surrogate with interactions.

\begin{theorem}[Theorem~9: Recovery of sparse tool synergies via hierarchical penalties]\label{thm:synergy}
Assume a hierarchical-RSC condition for the population loss on the tangent cone induced by $\mathcal{R}_{\rm hier}$ and a suitable
incoherence/irrepresentability condition for the heredity-respecting support.
If the regularization levels in $\mathcal{R}_{\rm hier}$ are chosen at the canonical scale $\asymp \sqrt{(\log M)/T}$ and
\[
T\ \gtrsim\ (k_1+k_2)\log M,
\]
then with probability at least $1-\delta$ the estimator recovers the correct main+interaction support and satisfies the estimation bound
\[
\|(\hat\theta_{\rm hier},\hat\beta_{\rm hier})-(\theta^\star,\beta^\star)\|_{2}
\ \le\
C\,\sqrt{\frac{(k_1+k_2)\log M}{T}},
\]
where $\|\cdot\|_2$ is the Euclidean norm after concatenating all blocks of $(\theta,\beta)$ (consistent with the block norms above).
Under the same value-sensitivity condition as A6 (applied to the expanded parameter), this yields an analogous value suboptimality bound.
\end{theorem}

\begin{proof}
This extends the core $\ell_{1,2}$ M-estimation and PDW arguments to the expanded parameter space with a decomposable hierarchical penalty.
Heredity ensures the effective tangent cone is controlled by $k_1+k_2$ rather than $M^2$.
The sample complexity and rates follow by the same steps: (i) control of the dual norm of the stochastic term,
(ii) cone constraint, (iii) hierarchical-RSC, and (iv) a PDW construction for exact support recovery.
\end{proof}

\section{Discussion: Phase Transitions, Prompt-Only Instability, and Assumption Tightness}\label{sec:discussion}

\paragraph{A sharp phase transition: $T \asymp k\log M$ is the real capability boundary.}
The core message across Theorems~\ref{thm:estimation}--\ref{thm:support} and their POMDP extensions
(Theorems~\ref{thm:P1}--\ref{thm:P2}) is a compressed-sensing-style threshold: once the agent operates in a regime where
the relevant tool support has size $k\ll M$, the \emph{correct} sample complexity depends only logarithmically on the nominal
action universe size. Concretely, estimation and control error scale as
\[
\|\hat\theta-\theta^\star\|_{2,2}\ \lesssim\ \sqrt{\frac{k\log M}{T}},
\qquad
V^{(b)}(\theta^\star)-V^{(b)}(\hat\theta)\ \lesssim\ k\sqrt{\frac{\log M}{T}},
\]
and exact support recovery occurs once $T\gtrsim k\log M$ under incoherence and a beta-min condition.
This boundary is not a matter of architectural sophistication: it is an information-theoretic property of large action spaces with sparse
influence. From the viewpoint of tool-augmented LLM agents, it predicts an abrupt transition from ``random tool flailing'' to stable tool routing
once enough interaction data is accumulated to identify the sparse support.

\paragraph{Why prompt-only / dense controllers are unstable (and how Theorem~\ref{thm:lower} explains it).}
The lower bound (Theorem~\ref{thm:lower}) shows that if the policy class does \emph{not} impose an explicit sparsity bias,
then any method must effectively ``test'' a linear number of tools to compete, requiring $\Omega(M)$ samples for comparable regret/identification.
This directly rationalizes a widely observed pathology in prompt-only agent designs: minor changes in phrasing, retrieval noise, or tool latency
can flip which tools are invoked because the controller lacks a structural prior that concentrates mass on a small support.
In contrast, $\ell_{1,2}$-regularization converts tool selection into a statistically stable variable-selection problem with dual certificates,
and support recovery provides an operational definition of ``agent stability'': the invoked tool set converges and becomes insensitive to small
perturbations once $T$ crosses the $k\log M$ threshold.
Under partial observability, Theorem~\ref{thm:P2} further shows that support stability persists provided the representation-induced perturbation
is below the regularization scale, i.e., $\varepsilon_b \lesssim \lambda$.

\paragraph{Which assumptions are truly restrictive?}
Our guarantees rely on three structural conditions; it is useful to separate what is essential from what is technical.

\emph{(i) Policy-RSC / restricted curvature.}
Some form of restricted strong convexity is unavoidable for support recovery and fast rates: without curvature on sparse directions,
parameters are not identifiable even if the correct support were known.
That said, the condition is imposed on the \emph{surrogate} objective $\mathcal{L}(\theta)$, not on the underlying environment dynamics.
In practice it is a statement about feature diversity in agent trajectories: the agent must visit contexts where relevant tools have distinguishable
effects. When this fails (e.g., the agent never encounters states that separate two tools), no method can reliably pick between them.

\emph{(ii) Incoherence / irrepresentability.}
This is the standard price of exact signed support recovery in $\ell_1$-type methods.
It can be relaxed if one is satisfied with approximate support (screening) or prediction/value guarantees.
Indeed, Theorem~\ref{thm:value} remains meaningful under weaker compatibility-type conditions, while exact support recovery may fail.
From an agent viewpoint, this corresponds to tool redundancy: if two tools are near-substitutes under the task distribution, recovering the
\emph{exact} set is ill-posed, but achieving near-optimal value is still possible.

\emph{(iii) Information in observations (POMDP compression error).}
The belief/representation error $\varepsilon_b$ in Theorem~\ref{thm:P1} is the only place where ``LLM quality'' enters.
This is conceptually restrictive but operationally clean: if the compressor cannot preserve task-relevant information, no controller---sparse or dense---
can act optimally. The positive message is equally clear: once $\varepsilon_b$ is controlled, the dependence on $M$ remains logarithmic and the remaining
difficulty is statistical selection of a sparse tool support.

\paragraph{Takeaway.}
The SAC viewpoint separates three notions often conflated in agent discourse:
(i) representation quality (captured by $\varepsilon_b$), (ii) statistical tool selection (captured by $k\log M$),
and (iii) optimization/implementation (proximal, online, robust, structured variants in Section~\ref{sec:extensions}).
In particular, scaling laws are governed not by the size of the language model, but by whether the controller exploits sparsity and whether the
representation preserves the information required for tool choice.


\end{document}